\documentclass[letterpaper, 10 pt, conference]{ieeeconf}
\IEEEoverridecommandlockouts    
\usepackage{cite}
\usepackage{amsmath,amssymb,mathrsfs,amsfonts}
\usepackage{bbm}
\usepackage{algorithmic}
\usepackage{graphicx}
\usepackage{subcaption}
\usepackage{color}

\newtheorem{theorem}{Theorem}

\newtheorem{remark}{Remark}

\begin{document}

\title{\LARGE \bf PDE-based Dynamic Control and Estimation of Soft Robotic Arms}
\author{Tongjia Zheng and Hai Lin
\thanks{*This work was supported by the National Science Foundation under Grant No. CNS-1830335, IIS-2007949.}
\thanks{Tongia Zheng and Hai Lin are with the Department of Electrical Engineering, University of Notre Dame, Notre Dame, IN 46556, USA (e-mail: tzheng1@nd.edu, hlin1@nd.edu.). }
}

\maketitle

\thispagestyle{empty}
\pagestyle{empty}

\begin{abstract}
Compared with traditional rigid-body robots, soft robots not only exhibit unprecedented adaptation and flexibility but also present novel challenges in their modeling and control because of their infinite degrees of freedom.
Most of the existing approaches have mainly relied on approximated models so that the well-developed finite-dimensional control theory can be exploited.
However, this may bring in modeling uncertainty and performance degradation. 
Hence, we propose to exploit infinite-dimensional analysis for soft robotic systems.
Our control design is based on the increasingly adopted Cosserat rod model, which describes the kinematics and dynamics of soft robotic arms using nonlinear partial differential equations (PDE).
We design infinite-dimensional state feedback control laws for the Cosserat PDE model to achieve trajectory tracking (consisting of position, rotation, linear and angular velocities) and prove their uniform tracking convergence.
We also design an infinite-dimensional extended Kalman filter on Lie groups for the PDE system to estimate all the state variables (including position, rotation, strains, curvature, linear and angular velocities) using only position measurements.
The proposed algorithms are evaluated using simulations.
\end{abstract}

% \begin{IEEEkeywords}
% Soft robots, Cosserat rod theory, PDE systems
% \end{IEEEkeywords}

%%%%%%%%%%%%%%%%%%%%%%%%%%%%%%%%%%%%%%%%%%%%%%%%%%%%%%%%%%%%%%%%%%%%%%%%%%%%%%%%
\section{Introduction}
Soft robotics is becoming an increasingly active research area in the realm of robotics.
Made of deformable materials, soft robots have the potential to exhibit unprecedented adaptation, sensitivity, and agility \cite{rus2015design}.
However, due to their infinite degrees of freedom, soft robotic systems also present novel challenges in their modeling and control.

There have been growing research activities in the modeling and control of soft robots.
Piecewise constant curvature (PCC) models are probably the most adopted strategy \cite{webster2010design}.
In this approach, a soft robotic arm is considered as consisting of a finite number of curved segments.
Each segment is represented by its curvature, arc length, and the angle of the plane containing the arc.
Then, the configuration space is approximated by a finite number of variables.
The PCC approach has produced fruitful results ranging from kinematic control to dynamic control in the last two decades \cite{webster2010design, marchese2016dynamics, della2020model}.
However, this constant curvature approximation is not always valid especially when the robot undergoes external loads and exhibits large deformation.
The finite element method (FEM) is another popular approximation approach for modeling and controlling soft robots \cite{duriez2013control, goury2018fast}, which represents the deformable shape as a set of mesh nodes together with the information of their neighbors.
While FEM can potentially model a wide range of geometric shapes, the major drawback lies in its computation cost.
Thus, control design based on FEM needs to rely on further approximations such as quasi-static assumptions, linearization, and model reduction \cite{goury2018fast}.

More accurate models are those derived from continuum mechanics, especially the Cosserat rod theory for rod-like soft robots \cite{antman2005nonlinear, rucker2011statics, renda2014dynamic}.
Cosserat rod theory describes the time evolution of the infinite-dimensional kinematic variables of a deformable rod undergoing external forces and moments using a set of nonlinear partial differential equations (PDE).
The previously mentioned PCC and FEM models, to some extent, can be viewed as finite-dimensional approximations of the Cosserat PDE models \cite{della2021model}.
Despite their modeling accuracy, nonlinear PDEs are known to be challenging for control purposes due to the lack of effective control design tools and analysis frameworks.
As a result, most existing control design based on Cosserat PDEs has relied on linearized \cite{shivakumar2021decentralized} or discretized models \cite{renda2018discrete, till2019real, grazioso2019geometrically, george2020first, doroudchi2021configuration}.
A notable exception is the energy shaping control presented in \cite{chang2020energy}.
% and designs feedback control laws to reach a desired equilibrium.

Over the past years, theoretical and experimental studies have suggested that feedback schemes are more robust to modeling uncertainties of soft robots \cite{della2021model}.
A feedback design usually requires intermediate states such as rotation, curvature, and velocity.
It is difficult and undesired to embed various types of sensors into soft robots because they will undermine the inherent softness and bring in additional modeling errors.
The more desired approach is to measure the states that are easily available and use observers to infer other intermediate states.
Compared with modeling and control, the state estimation problem of soft robots has remained largely unexplored until recently \cite{lunni2018shape, thieffry2018control, loo2019h}.
Nevertheless, the estimation problem studied therein is based on simplified or approximated models.

% It is observed that most existing works have followed an Approximate-then-Design (AtD) approach which finds a finite-dimensional approximation (using PCC, FEM, or discretized Cosserat models) and then designs control and estimation algorithms.
% While this is usually the easier way, it brings in problems of modeling uncertainty and performance degradation \cite{della2021model}.
% Hence, we choose to follow a Design-then-Approximate (DtA) approach which designs infinite-dimensional algorithms and then approximates their final forms in implementation.
% In this approach, the established stability properties tend to be robust because the approximation errors appear as additive disturbances \cite{zheng2021transporting}.

In this work, our control design is based directly on the nonlinear Cosserat PDE models.
Besides nonlinearity, another challenge lies in that the Cosserat PDE evolves on $SE(3)$, the special Euclidean group, which is not a vector space.
We address this challenge by extending geometric control \cite{lee2010geometric} and estimation on Lie groups \cite{barfoot2017state} to infinite-dimensional systems.
Assuming full actuation, we design state feedback control to achieve trajectory tracking in the task space (consisting of position, rotation, linear and angular velocities), and prove their uniform tracking convergence.
To estimate the feedback states, we show how to linearize the Cosserat PDE on Lie groups using exponential maps and design an infinite-dimensional extended Kalman filter \cite{bensoussan1971filtrage} to estimate all the state variables (including position, rotation, strains, curvature, velocities) using only position measurements.
Leveraging this work, we aim to make a step towards the exploitation of infinite-dimensional control and estimation theory for soft robotic systems.

The remainder of the paper is organized as follows.
The Cosserat rod model is introduced in Section \ref{section:modeling}. 
In Section \ref{section:control}, we design infinite-dimensional state feedback controllers and prove their stability. 
In Section \ref{section:estimation}, we design infinite-dimensional extended Kalman filters for the Cosserat PDE model.
Section \ref{section:simulation} presents simulations to verify the effectiveness of the algorithms.
Section \ref{section:conclusion} summarizes the contribution and points out future research.

\section{modeling of soft robotic arms}
\label{section:modeling}
% Denote by $\mathbb{R}^n$ the $n$-dimensional Euclidean space.
The special orthogonal group is denoted by $SO(3)=\{R\in\mathbb{R}^{3\times3}\mid R^TR=I,\operatorname{det}R=1\}$.
The associated Lie algebra is the set of skew-symmetric matrices $\mathfrak{so}(3)=\{A\in\mathbb{R}^{3\times3}\mid A=-A^T\}$.
Define the hat operator $(\cdot)^\wedge:\mathbb{R}^3\to\mathfrak{so}(3)$ by the condition that $u^\wedge v=u\times v$ for all $u,v\in\mathbb{R}^3$, where $\times$ denotes the cross product.
% In other words,
% \begin{align*}
%     u^\wedge=
%     \begin{bmatrix}
%     0 & -u_{z} & u_{y} \\
%     u_{z} & 0 & -u_{x} \\
%     -u_{y} & u_{x} & 0
%     \end{bmatrix}
% \end{align*}
% for $u=[u_x,u_y,u_z]^T\in\mathbb{R}^3$. 
Let $(\cdot)^\vee:\mathfrak{so}(3)\to\mathbb{R}^3$ be its inverse operator, i.e., $(u^\wedge)^\vee=u$.
% For a matrix-valued function $f:\mathbb{R}^l\to\mathbb{R}^{m\times n}$ where $s,t\in\mathbb{R}$, denote 
% \begin{align*}
%     f_{x_i}(x)=
%     \begin{bmatrix}
%     \frac{\partial f_{1,1}(x)}{\partial x_i} & \dots & \frac{\partial f_{1,n}(x)}{\partial x_i} \\
%     \vdots & \ddots & \vdots \\
%     \frac{\partial f_{m,1}(x)}{\partial x_i} & \dots & \frac{\partial f_{m,n}(x)}{\partial x_i}
%     \end{bmatrix}
% \end{align*}
% where $x_i$ is the $i$-th element of $x$ and $f_{j,k}$ is the $(i,j)$-th element of $f$.

Cosserat rod theory describes the dynamic response of a long and thin deformable rod undergoing external forces and moments, and is widely adopted to model soft robotic arms \cite{rucker2011statics, renda2014dynamic}.
In Cosserat rod theory, a rod is represented as a curve in space and the state variables are defined as functions of time $t\in\mathbb{R}$ and an arc length parameter $s\in \mathbb{L}=[0~L]$ along the undeformed rod centerline, where $L$ is the total length; see Fig. \ref{fig:Cosserat rod}.
(It is important to emphasize that the arc length parameter is defined along the undeformed rod because the total rod length may change after deformation.)
Thus, the pose of a rod is uniquely determined by the position $p(s,t)\in\mathbb{R}^{3}$ and orientation $R(s,t)\in SO(3)$ of every cross section at $s$ in the global frame.
Since every cross section can be associated with a local frame, a rod has infinite degrees of freedom.
In the local frames, denote by $\nu(s,t)\in\mathbb{R}^3$ the linear velocity, $\omega(s,t)\in\mathbb{R}^3$ the angular velocity, $q(s,t)\in\mathbb{R}^3$ the linear strains (for shear and extension), and $u(s,t)\in\mathbb{R}^3$ the angular strains (for bending/curvature and torsion); see Fig. \ref{fig:strains}.
% $(p,R,\nu,\omega)$ are the kinematic variables in the task space.
% $(q,u)$ are the configuration variables.
Finally, in the global frame, denote by $n(s,t)\in\mathbb{R}^{3}$ and $m(s,t)\in\mathbb{R}^{3}$ the internal forces and moments, respectively, and $f(s,t)\in\mathbb{R}^3$ and $l(s,t)\in\mathbb{R}^3$ the external/distributed forces and moments, respectively.
% Assume the rod is homogeneous and inextensible (i.e., no longitudinal strain).
The dynamic response of a Cosserat rod is characterized by the following set of PDEs \cite{antman2005nonlinear,simo1988dynamics}:
\begin{align}
    p_s& =Rq \label{eq:p_s}\\
    R_s& =Ru^\wedge \label{eq:R_s}\\
    p_t& =R\nu \label{eq:p_t local}\\
    R_t& =R\omega^\wedge \label{eq:R_t}\\
    n_s+f& =\rho \sigma p_{tt} \label{eq:linear motion equation}\\
    m_s+p_s\times n+l& =(\rho RJ\omega)_t \label{eq:angular motion equation}\\
    n& =RK_l\left(q-\bar{q}\right) \label{eq:linear constitutive law}\\
    m& =RK_a\left(u-\bar{u}\right) \label{eq:angular constitutive law}
\end{align}
where $(\cdot)_t:=\frac{\partial}{\partial t}(\cdot)$ and $(\cdot)_s:=\frac{\partial}{\partial s}(\cdot)$ represent partial derivatives, $\rho(s)\in\mathbb{R}$ is the density of the rod, $\sigma(s)\in\mathbb{R}$ is the cross-sectional area, $J(s)\in\mathbb{R}^{3\times3}$ is the rotational inertia matrix, $\bar{q}(s)$ and $\bar{u}(s)$ are the undeformed values of $q$ and $u$.
In the case of straight reference configuration, $\bar{q}(s)=[0~0~1]^T$ and $\bar{u}=[0~0~0]^T$.
$K_l(s)\in\mathbb{R}^{3\times3}$ and $K_a(s)\in\mathbb{R}^{3\times3}$ are the linear and angular stiffness matrices given by
\begin{equation*}
    K_l(s)=\operatorname{diag}(G,G,E)\sigma(s), \quad K_a(s)=\operatorname{diag}(E,E,G)J(s),
\end{equation*}
where $E\in\mathbb{R}$ and $G\in\mathbb{R}$ are the Young's and shear moduli.
% \begin{align*}
%     J=\operatorname{diag}(J_x,J_y,J_z), \quad J_x=J_y=\pi r^4/4, \quad J_z=J_x+J_y.
% \end{align*}
Assume the material parameters are constant and known.
Also assume $f=f_e+f_c$ and $l=l_e+l_c$ where $f_e(s,t)\in\mathbb{R}^3$ and $l_e(s,t)\in\mathbb{R}^3$ are known environment forces (e.g., gravity) and moments, $f_c(s,t)\in\mathbb{R}^3$ and $l_c(s,t)\in\mathbb{R}^3$ are control inputs.

\begin{figure}[t]
    \centering
    \begin{subfigure}[b]{0.8\columnwidth}
        \centering
        \includegraphics[width=\textwidth]{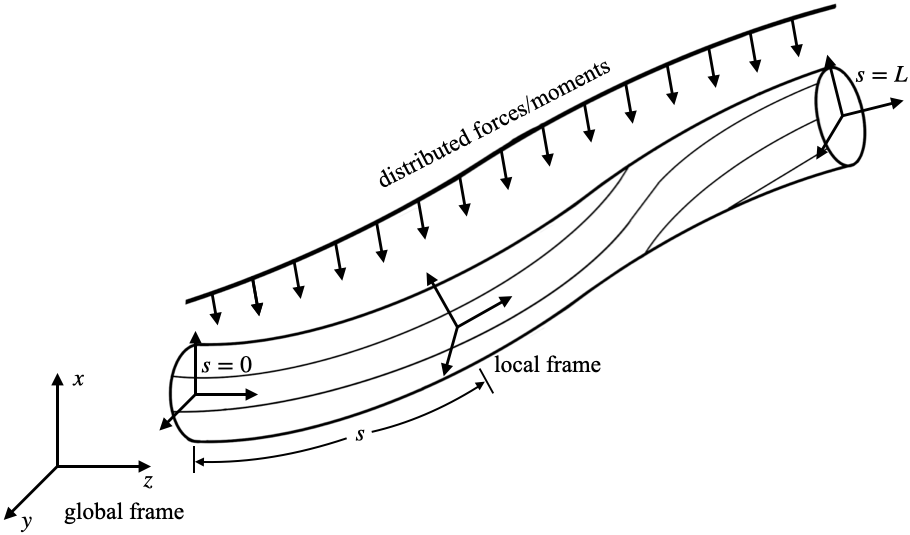}
        \caption{A Cosserat rod.}
        \label{fig:Cosserat rod}
    \end{subfigure}
    
    \begin{subfigure}[b]{0.6\columnwidth}
        \centering
        \includegraphics[width=\textwidth]{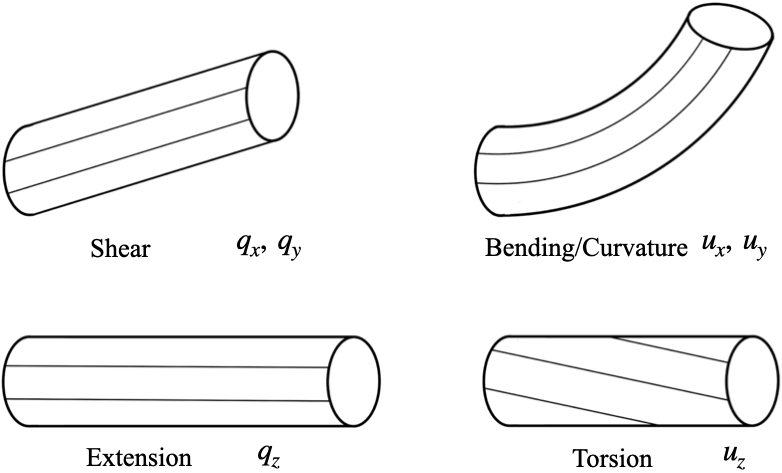}
        \caption{Four types of strains.}
        \label{fig:strains}
    \end{subfigure}
    \caption{Illustration of Cosserat rods.}
\end{figure}
% Note that in general $R_t,R_s\notin SO(3)$ or $\mathfrak{so}(3)$.
% $u$ uniquely determines $R$ if the boundary value of $R$ is given. Then $q$ uniquely determines $p$ if the boundary value of $p$ is given. Hence, $u$ and $q$ uniquely determine $p$ and $R$.

Equations \eqref{eq:p_s}-\eqref{eq:R_t} are the kinematic equations.
Equations \eqref{eq:linear motion equation}-\eqref{eq:angular motion equation} are the dynamic equations.
Equations \eqref{eq:linear constitutive law}-\eqref{eq:angular constitutive law} are the constitutive laws where linear constitutive laws are assumed.
Depending on the material, other forms of constitutive laws may be adopted.
% Using the fact that $p_{st}=p_{ts}$ and $R_{st}=R_{ts}$, we can derive the following compatibility equations which account for the time evolution of the strains $q$ and $u$:
% \begin{align}
%     q_t= &\nu_s+u^\wedge\nu-\omega^\wedge q \label{eq:compatibility equation 1}\\
%     u_t= &\omega_s+u^\wedge\omega. \label{eq:compatibility equation 2}
% \end{align}
Note that there are only four independent state variables in \eqref{eq:p_s}-\eqref{eq:angular constitutive law}.
% We choose the kinematic variables $(p,R,\nu,\omega)$ as system states.
% By substituting \eqref{eq:linear constitutive law}-\eqref{eq:angular constitutive law} into \eqref{eq:linear motion equation}-\eqref{eq:angular motion equation}, we obtain the first form of PDE models for soft robotic arms as follows.
% \begin{align} \label{eq:PDE system 1}
% \begin{split}
%     p_t= & R\nu \\
%     R_t= & R\omega^\wedge \\
%     % q_t= & \nu_s+u^\wedge\nu+q^\wedge\omega \\
%     % u_t= & \omega_s+u^\wedge\omega \\
%     \begin{split}
%     \nu_t= & \frac{1}{\rho\sigma}\big(K_l(q_s-\bar{q}_s)+u^\wedge K_l(q-\bar{q}) \\
%     & -\rho\sigma\omega^\wedge\nu+R^T(f_e+f_c)\big)
%     \end{split} \\
%     \begin{split}
%         \omega_t= & (\rho J)^{-1}\big(K_a(u_s-\bar{u}_s)+u^\wedge K_a(u-\bar{u}) \\
%         & +q^\wedge K_l(q-\bar{q})-\omega^\wedge\rho J\omega+R^T(l_e+l_c)\big),
%     \end{split}
% \end{split}
% \end{align}
% where $q$ and $u$ should be understood as mappings of $R$ and $p$ according to $q=R^Tp_s$ and $u=(R^TR_s)^\vee$.
% This is the model derived in \cite{rucker2011statics, renda2014dynamic}.
For control purposes, one will find it more convenient to work with the global linear velocity, which is denoted by $v(s,t)\in\mathbb{R}^3$ and satisfies $v=R\nu$.
By substituting \eqref{eq:linear constitutive law}-\eqref{eq:angular constitutive law} into \eqref{eq:linear motion equation}-\eqref{eq:angular motion equation} and using $v_t=R_t\nu+R\nu_t$, we obtain the PDE model with minimum number of states:
\begin{align} \label{eq:PDE system}
\begin{split}
    p_t= & v \\
    R_t= & R\omega^\wedge \\
    % q_t= & R_s^Tv+R^Tv_s+u^\wedge R^Tv+q^\wedge\omega \\
    % u_t= & \omega_s+u^\wedge\omega \\
    v_t= & \frac{1}{\rho\sigma}\big(RK_l(q_s-\bar{q}_s)+R_sK_l(q-\bar{q})+f_e+f_c\big) \\
    \begin{split}
        \omega_t= & (\rho J)^{-1}\big(K_a(u_s-\bar{u}_s)+u^\wedge K_a(u-\bar{u}) \\
        & +q^\wedge K_l(q-\bar{q})-\omega^\wedge\rho J\omega+R^T(l_e+l_c)\big),
    \end{split}
\end{split}
\end{align}
which can be seen as a change of coordinates of the models derived in \cite{rucker2011statics, renda2014dynamic}.
This change of coordinates, however, makes the control design and stability analysis easier, which will be seen later.
Finally, assuming the soft robotic arm has a fixed end and a free end, the boundary conditions are given by
\begin{align*}
    & p(0,t)=p_0,\quad R(0,t)=R_0, \\
    % & n(L,t)=[0~0~0]^T, \quad m(L,t)=[0~0~0]^T. \\
    % & {\color{blue}q(L,t)=\bar{q}(L), \quad u(L,t)=\bar{u}(L).}
    & n(L,t)=0, \quad m(L,t)=0.
    % & p_s(L,t)=0,\quad R_s(L,t)=0.
\end{align*}

\section{Control design of soft robotic arms}
\label{section:control}
In this section, we design infinite-dimensional state feedback control for the soft robotic arm to track a desired trajectory in task space.
Assume that a smooth desired trajectory consisting of $p_*(s,t)\in\mathbb{R}^3$ and $R_*(s,t)\in SO(3)$ (alternatively written as $p^*$ and $R^*$ when the subscript position is needed for other notations) is given and satisfies:
\begin{align} \label{eq:desired trajectory}
    p_t^*=v_*,\qquad R_t^*=R_*\omega_*^\wedge,
\end{align}
where $v_*(s,t)\in\mathbb{R}^3$ and $\omega_*(s,t)\in\mathbb{R}^3$ (alternatively $v^*$ and $\omega^*$) are the desired global linear velocity and local angular velocity, respectively, which are assumed to be smooth and uniformly bounded.
Define the following error terms:
\begin{align*}
e_p &=p-p_*, & e_R &=\frac{1}{2}(R_*^TR-R^TR_*)^\vee, \\
e_v &=v-v_*, & e_\omega &=\omega-R^TR_*\omega_*.
\end{align*}
Note that $R_*^TR\to I$ when $R\to R_*$.
We aim to design control inputs $f_c$ and $l_c$ such that all the error terms converge to 0.

% We assume that the states $(p,R,v,\omega)$ are available.
% % As a result, $(q,u)$ are also available using \eqref{eq:p_s}-\eqref{eq:R_s}.
% This enables us to design state feedback control.
The first step is a feedforward transform to cancel the nonlinear dynamics.
Noting that $R^{-1}=R^T$, we let
\begin{align} \label{eq:feedforward transform}
\begin{split}
    f_c= & -RK_l(q_s-\bar{q}_s)-R_sK_l(q-\bar{q})+\rho\sigma f_*-f_e \\
    l_c= & -R\Big[K_a\big(u_s-\bar{u}_s\big)+u^\wedge K_a(u-\bar{u}) \\
    & +q^\wedge K_l(q-\bar{q})-\omega^\wedge\rho J\omega-\rho Jl_*\Big]-l_e.
\end{split}
\end{align}
% \begin{align}
%     f_c= & -RK_l\big((R^Tp_s)_s-\bar{q}_s\big)-R_sK_l(R^Tp_s-\bar{q}) \nonumber\\
%     & +\rho\sigma f_*-f_e \label{eq:f feedforward}\\
%     l_c= & -R\Big[K_a\big((R^TR_s)_s^\vee-\bar{u}_s\big)+R^TR_sK_a\big((R^TR_s)^\vee-\bar{u}\big) \nonumber\\
%     & +(R^Tp_s)^\wedge K_l(R^Tp_s-\bar{q})-\omega^\wedge\rho J\omega-\rho Jl_*\Big]-l_e. \label{eq:l feedforward}
% \end{align}
Substituting \eqref{eq:feedforward transform} into \eqref{eq:PDE system}, we obtain:
\begin{align}
    p_t= & v \label{eq:p}\\
    v_t= & f_* \label{eq:v}\\
    R_t= & R\omega^\wedge \label{eq:R}\\
    \omega_t= & l_*. \label{eq:omega}
\end{align}

\begin{remark}
We note that the transformed system \eqref{eq:p}-\eqref{eq:omega} is no longer a PDE.
Instead, it is an infinite-dimensional (or parametrized) ODE.
One may recognize that for any fixed $s\in\mathbb{L}$, the transformed system resembles a fully actuated control system on $SE(3)$.
This important observation suggests that many control design techniques on $SE(3)$ (such as those developed for quadrotors \cite{lee2010geometric}) can potentially be extended to tackle the control problem of soft robotic arms.
% It also suggests that we may be able to control all the states simultaneously, which is impossible for quadrotors.
We also note that the position control problem \eqref{eq:p}-\eqref{eq:v} and rotation control problem \eqref{eq:R}-\eqref{eq:omega} are independent after the feedforward transform.
This is the advantage of using the global linear velocity as a state.
The decoupling also makes stability analysis easier.
\end{remark}

The position control \eqref{eq:p}-\eqref{eq:v} is simply a linear control problem, for which we design $f_*$ as
\begin{equation} \label{eq:f*}
    f_*=v_t^*-k_pe_p-k_ve_v,
\end{equation}
where $k_p(s),k_v(s)\in\mathbb{R}$ are smooth positive functions.
% are computed from the following Lyapunov equation (parametrized by $s$)
% \begin{equation*}
% P(A-BK)+(A-BK)^TP=-Q,
% \end{equation*}
% where $P(s),Q(s)\in\mathbb{R}^{3\times3}$ are positive definite, and $A,B,K$ are given by
% \begin{equation*}
%     A=\begin{bmatrix}
%     0 & I \\
%     0 & 0
%     \end{bmatrix},\quad
%     B=\begin{bmatrix}
%     0 \\
%     I
%     \end{bmatrix},\quad
%     K=\begin{bmatrix}
%     k_pI & k_vI
%     \end{bmatrix}.
% \end{equation*}
% The Lyapunov equation always has a solution because the pair $(A,B)$ is controllable.
For the rotation control problem, we design $l_*$ as 
\begin{equation} \label{eq:l*}
    l_*=R^TR_*\omega_t^*-k_Re_R-k_\omega e_\omega-\omega^\wedge R^TR_*\omega_*,
\end{equation}
where $k_{R}(s),k_{\omega}(s)\in\mathbb{R}$ are smooth positive functions.
This controller can be seen as a parametrization (by $s$) of the geometric controller in \cite{lee2010geometric}.
We allow the gains $(k_p,k_v,k_R,k_\omega)$ to be functions of $s$ which can provide flexibility to the implementation of these control laws.

We have the following uniform (in $s$) convergence result.

\begin{theorem}
\label{thm:control}
Consider the PDE system \eqref{eq:PDE system}.
Let $(f_c,l_c)$ be designed as \eqref{eq:feedforward transform} where $(f_*,l_*)$ are given by \eqref{eq:f*}-\eqref{eq:l*}.
If the initial conditions satisfy, for $\forall s$, 
\begin{align}
    & \operatorname{tr}[I-R_*(s,0)^TR(s,0)]<4, \label{eq:constraint 1}\\
    & \|e_\omega(s,0)\|^2<k_R(s)(4-\operatorname{tr}[I-R_*(s,0)^TR(s,0)]), \label{eq:constraint 2}
\end{align}
then as $t\to\infty$, $\big(e_p(s,t),e_v(s,t),e_R(s,t),e_\omega(s,t)\big)\to0$ for $\forall s$ exponentially.
\end{theorem}

\begin{proof}
The proof requires some arguments in the proof of Proposition 1 in \cite{lee2010geometric}, which will be included for completeness.
Substituting \eqref{eq:desired trajectory} and \eqref{eq:f*}-\eqref{eq:l*} into \eqref{eq:p}-\eqref{eq:omega}, we obtain:
\begin{align*}
    \partial_te_p & =e_v, \\
    \partial_te_v & =-k_pe_p-k_ve_v \\
    \partial_te_R & =\frac{1}{2}(R_*^TRe_\omega^\wedge+e_\omega^\wedge R^TR_*)^\vee \\
    & =\frac{1}{2}(\operatorname{tr}[R^TR_*]I-R^TR_*)e_\omega \\
    & =:C(R_*^TR)e_\omega \\
    \partial_te_\omega & =-k_Re_R-k_\omega e_\omega
\end{align*}
where we used the identity $x^\wedge A+A^Tx^\wedge=\big((\operatorname{tr}[A]I-A)x\big)^\wedge$ for $A\in\mathbb{R}^{3\times3}$, $x\in\mathbb{R}^3$.
It is known that $\|C(R_*^TR)\|\leq1$ \cite{lee2010geometric}.
Consider a parametrized Lyapunov function $V(s,t)=V_1(s,t)+V_2(s,t)$ with
\begin{align*}
    V_1(s,t) & =
    \begin{bmatrix}
    e_p \\
    e_v
    \end{bmatrix}^TP(s)
    \begin{bmatrix}
    e_p \\
    e_v
    \end{bmatrix} \\
    V_2(s,t) & =\frac{1}{2}k_R(s)\operatorname{tr}[I-R_*^TR]+\frac{1}{2}e_\omega^Te_\omega+c(s)e_R^Te_\omega
\end{align*}
where $P(s)\in\mathbb{R}^{6\times6}$ and $c(s)\in\mathbb{R}$ will be determined later.
Since $(e_p,e_v)$ satisfy a linear system, it is easy to show that there exists positive definite $P(s)$ such that $\partial_tV_1$ is negative definite for $\forall s$.
For $V_2$, it is known that \cite{lee2010geometric}
\begin{align*}
    e_2^TM_1(s)e_2\leq V_2\leq e_2^TM_2(s)e_2
\end{align*}
for $\forall s$, where $e_2(s,t)=[\|e_R(s,t)\|,\|e_\omega(s,t)\|]^T$ and
\begin{align*}
    M_1(s)=\frac{1}{2}
    \begin{bmatrix}
    k_R(s) & -c(s) \\
    -c(s) & 1
    \end{bmatrix}, \quad M_2(s)=\frac{1}{2}
    \begin{bmatrix}
    \frac{2k_R(s)}{2-d(s)} & c(s) \\
    c(s) & 1
    \end{bmatrix}
\end{align*}
where $d(s):=V_2(s,0)/k_R(s)\mid_{c=0}$ and $d(s)<2$ if \eqref{eq:constraint 1}-\eqref{eq:constraint 2} hold.
Next,
\begin{align*}
    \partial_tV_2 & =-ck_Re_R^Te_R-\big(k_\omega-cC(R_*^TR)\big)e_\omega^Te_\omega-ck_\omega e_R^Te_\omega.
\end{align*}
Since $\|C(R_*^TR)\|\leq1$, if we choose
\begin{align*}
    0<c(s)<\min\Big\{k_\omega(s),\frac{4k_R(s)k_\omega(s)}{k_\omega^2(s)+4k_R(s)},\sqrt{k_R(s)}\Big\},
\end{align*}
$V_2$ is positive definite and $\partial_tV_2$ is negative definite for $\forall s$.
\end{proof}

In this theorem, assumption \eqref{eq:constraint 1} almost always holds because by definition $\operatorname{tr}[I-R_*(s,0)^TR(s,0)]\leq4$ and the identity holds only when one of the initial angular errors is $\pi$ rad which has zero measure.
Assumption \eqref{eq:constraint 2} holds as long as we choose a sufficiently large $k_R(s)$.

\begin{remark}
The control law in Theorem \ref{thm:control} has been designed assuming full actuation.
In the implementation, it needs to be approximated because we can only place a finite number of actuators.
This allocation problem highly depends on the actuators and is under study.
% We expect the stability properties established here to be robust even if approximated inputs are applied because the induced approximation error will enter the system as additive disturbance input and can be studied using tools like input-to-state stability \cite{zheng2021transporting}.
% This is the advantage of adopting the DtA approach.
% For control problems, infinite-dimensional analysis can result in more compact and interpretable solutions, whose stability properties tend to be robust even if approximated inputs are applied, while the AtD approach may cause instability because of the ignored higher-order dynamics in the approximated model.
\end{remark}

\section{State estimation of soft robotic arms}
\label{section:estimation}
% Observability depends on $q_s$ and $q$. If they don't change, then $R$ is not observable.
This section studies the state estimation problem.
Assume the output is a noisy position measurement given by
\begin{equation*}
    y(s,t)=p(s,t)+w(s,t),
\end{equation*}
where $w(s,t)\in\mathbb{R}^3$ is the measurement noise assumed to be Gaussian with covariance operator $\mathcal{R}(t)$.
Such a position measurement can be obtained using stereo cameras.

\begin{remark}
We present a brief discussion about the observability problem in this setup.
First, $v$ can be estimated from $p_t$.
According to the third equation of \eqref{eq:PDE system}, $R$ can be estimated from $(v_t,p)$, or $(p_{tt},p_t,p)$.
Finally, $\omega$ can be estimated from $R_t$, or $(p_{ttt},p_{tt},p_t,p)$.
The key step to estimating the angular states from linear states lies in the third equation of \eqref{eq:PDE system}.
Depending on the soft materials, if $(q-\bar{q})$, the change of linear strains, is negligible, the linear and angular states would be almost independent of each other.
In this case, additional angular measurements (of either $R$, $u$, or $\omega$) will be needed for more accurate estimation of the angular states.
\end{remark}

Since the Cosserat rod PDE \eqref{eq:PDE system} is nonlinear, we propose to design an infinite-dimensional extended Kalman filter \cite{bensoussan1971filtrage}.
The challenge lies in that $SO(3)$ is not a vector space.
Hence, an approximation like $R'\approx R+\delta R$ with $R,\delta R\in SO(3)$ may not make sense because $R'$ may not belong to $SO(3)$.
% This brings difficulties to the linearization problem.
% Linearization of ODEs evolving on $SO(3)$ (or more generally, Lie groups) has been studied in \cite{barfoot2017state}.
% We will generalize these ideas to PDEs evolving on $SO(3)$.
Fortunately, linearization for equations on $SO(3)$ (or more generally, Lie groups) has been studied in recent years.
We follow the strategy in Section 7.2.3 of \cite{barfoot2017state}, with appropriate generalization to infinite-dimensional Lie groups.

The key is to relate elements of $SO(3)$ (a Lie group) to elements of $\mathfrak{so}(3)$ (a vector space) using the exponential map:
\begin{equation*}
    R=\exp(\eta^\wedge)=\sum_{n=0}^{\infty}\frac{1}{n!}(\eta^\wedge)^n,
\end{equation*}
where $R(s,t)\in SO(3)$ and $\eta(s,t)\in\mathbb{R}^3$ (and hence $\eta^\wedge(s,t)\in\mathfrak{so}(3)$). 
Denote by $\delta(\cdot)$ an infinitesimal perturbation.
For a perturbed rotation matrix $R'(s,t)$, we have
\begin{equation*}
    R'=R\exp(\delta\eta^\wedge)\approx R(I+\delta\eta^\wedge)=R+R\delta\eta^\wedge\in SO(3),
\end{equation*}
where $R(s,t)$ is the nominal solution and $\delta\eta(s,t)$ is an infinitesimal perturbation acting as a rotation vector.
This linearization scheme suggests that we can take $\delta R\approx R\delta\eta^\wedge$.

Now we linearize the second equation of \eqref{eq:PDE system}.
Let $\omega'(s,t)$ be the perturbed angular velocity.
Then, the perturbed kinematic equation is given by $R_t'=R'(\omega')^\wedge$ which can be approximated as follows by dropping the higher-order terms.
\begin{align*}
    \big(R(I+\delta\eta^\wedge)\big)_t & \approx R(I+\delta\eta^\wedge)(\omega+\delta\omega)^\wedge \\
    R\omega^\wedge+R\omega^\wedge\delta\eta^\wedge+R\delta\eta_t^\wedge & \approx R\omega^\wedge+R\delta\eta^\wedge\omega^\wedge+R\delta\omega^\wedge \\
    \delta\eta_t & \approx-\omega^\wedge\delta\eta+\delta\omega,
\end{align*}
where we used the identity $x^\wedge y^\wedge-y^\wedge x^\wedge=(x^\wedge y)^\wedge$.
Thus, we obtain a linearized equation of $\eta(s,t)$ on a vector space, which is used to replace the second equation in \eqref{eq:PDE system}.
We linearize other equations in \eqref{eq:PDE system} using this perturbation scheme.
The complete derivation is included in the Appendix.

Denote $\xi=(p,\eta,v,\omega)\in\mathbb{R}^{12}$ and let $\hat{\xi}=(\hat{p},\hat{\eta},\hat{v},\hat{\omega})$ be an estimate of $\xi$.
(We distinguish between $\hat{(\cdot)}$ and $(\cdot)^\wedge$ where the former represents an estimate and the latter is the hat operator.)
For generality, we present the result when the equation of $\xi$ is disturbed by Gaussian noise with covariance operator $\mathcal{Q}(t)$, although in our problem $\mathcal{Q}(t)=0$.
% (The rigorous treatment of process noise requires definitions of Gaussian random variables on Lie groups \cite{barfoot2017state} which will be studied in future work.)
The extended Kalman filter for $\xi$ is given by:
\begin{equation}
    \partial_t\hat{\xi}=\mathcal{F}(\hat{\xi},f_c,l_c)+\mathcal{K}(t)(y-\mathcal{C}\hat{\xi})
\end{equation}
where $\mathcal{F}$ is the nonlinear operator that represents the original dynamics in \eqref{eq:PDE system}, $\mathcal{K}(t)=\mathcal{P}(t)\mathcal{C}^*\mathcal{R}^{-1}(t)$ is the Kalman gain operator, and $\mathcal{P}(t)$ is the solution of the following infinite-dimensional Riccati equation:
\begin{equation*}
    \dot{\mathcal{P}}(t)=\mathcal{A}_{\hat{\xi}}\mathcal{P}(t)+\mathcal{P}(t)\mathcal{A}_{\hat{\xi}}^*+\mathcal{Q}(t)-\mathcal{P}(t)\mathcal{C}^*\mathcal{R}^{-1}(t) \mathcal{C}\mathcal{P}(t),
\end{equation*}
where $(\cdot)^*$ represents the adjoint operator.
(See the Appendix for the definition of these operators.)
The Kalman gain operator has the structure:
\begin{equation*}
    \mathcal{K}=\begin{bmatrix}
    \mathcal{K}_1 \\
    \vdots \\
    \mathcal{K}_4
    \end{bmatrix}
\end{equation*}
where $\mathcal{K}_i,i=1,\dots,4,$ are the corresponding gain operators for the four components of $\hat{\xi}$.
Then, the explicit form of the extended Kalman filter in terms of $\hat{R}$ is given by
\begin{align*}
\begin{split}
    \hat{p}_t= & \hat{v}+\mathcal{K}_1(y-\hat{p}) \\
    \hat{R}_t= & \hat{R}\big(\hat{\omega}+\mathcal{K}_2(y-\hat{p})\big)^\wedge \\
    \hat{v}_t= & \frac{1}{\rho\sigma}\Big[\hat{R}K_l(\hat{q}_s-\bar{q}_s)+\hat{R}_sK_l(\hat{q}-\bar{q})+(f_e+f_c)\Big] \\
    & +\mathcal{K}_3(y-\hat{p}) \\
    \hat{\omega}_t= & (\rho J)^{-1}\Big[K_a(\hat{u}_s-\bar{u}_s)+\hat{u}^\wedge K_a(\hat{u}-\bar{u}) \\
    & +\hat{q}^\wedge K_l(\hat{q}-\bar{q})-\hat{\omega}^\wedge\rho J\hat{\omega}+\hat{R}^T(l_e+l_c)\Big]+\mathcal{K}_4(y-\hat{p}),
\end{split}
\end{align*}
where we notice that when converted from $\hat{\eta}$ to $\hat{R}$, the innovation term $\mathcal{K}_2(y-\hat{p})$ is inserted into $\hat{\omega}$ instead of being added on the right-hand side.
Finally, to obtain estimates for $q$ and $u$, simply use $\hat{q}=\hat{R}^T\hat{p}_s$ and $\hat{u}=(\hat{R}^T\hat{R}_s)^\vee$.

\begin{figure*}[t]
    \centering
    \begin{subfigure}[b]{0.19\textwidth}
        \centering
        \includegraphics[width=\textwidth]{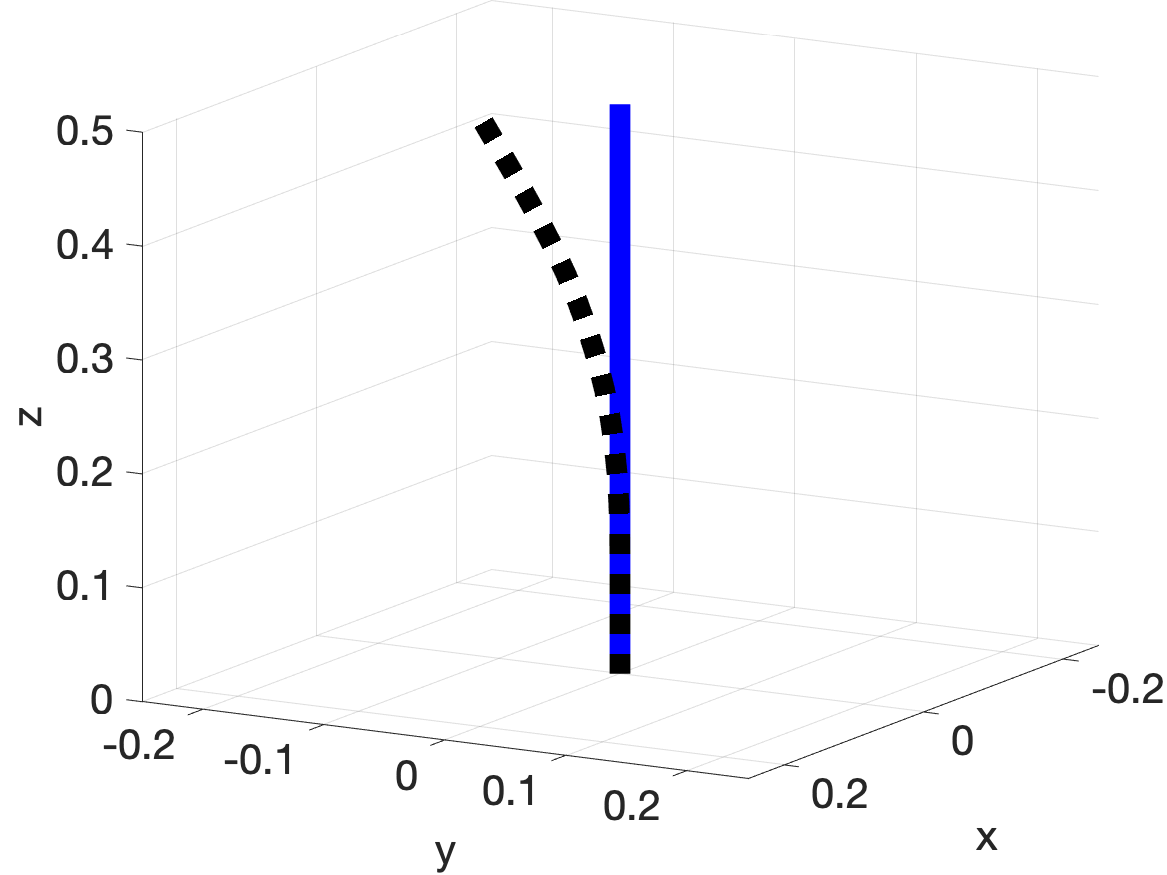}
    \end{subfigure}
    \begin{subfigure}[b]{0.19\textwidth}
        \centering
        \includegraphics[width=\textwidth]{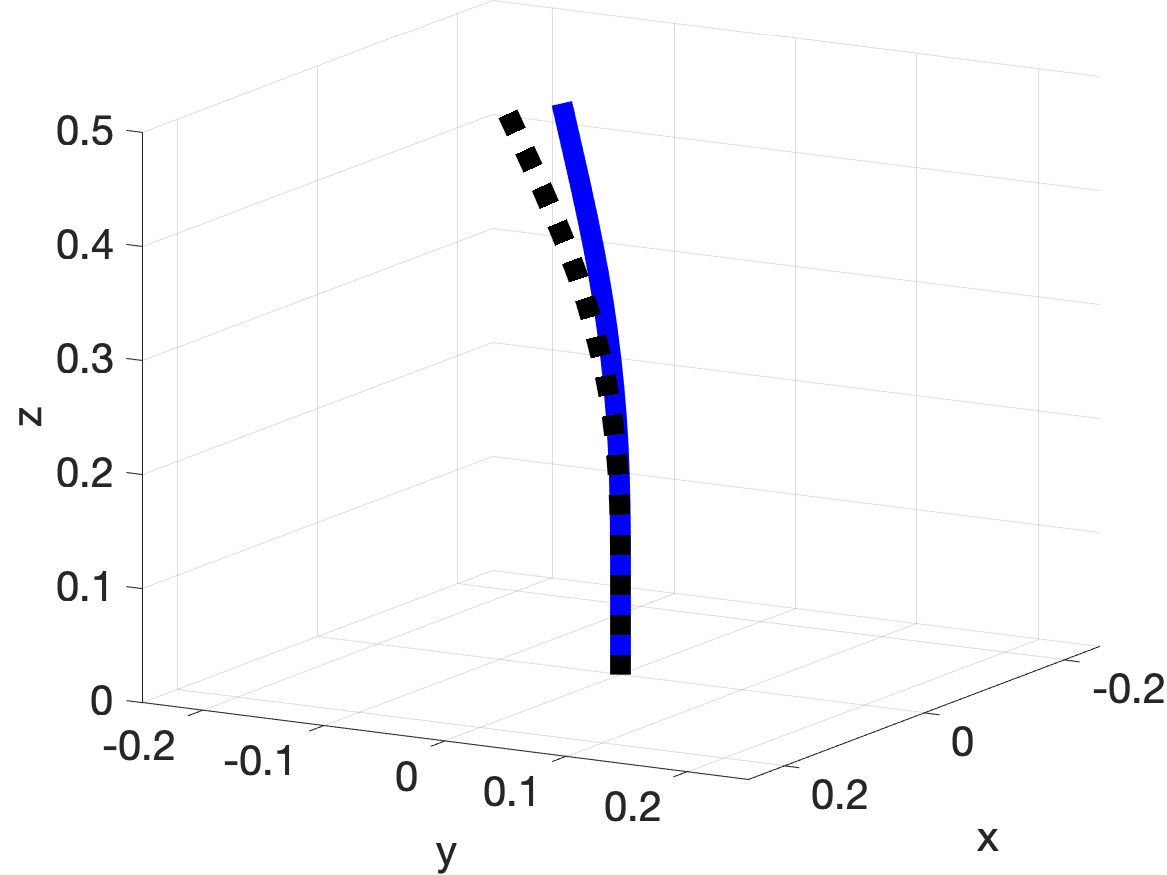}
    \end{subfigure}
    \begin{subfigure}[b]{0.19\textwidth}
        \centering
        \includegraphics[width=\textwidth]{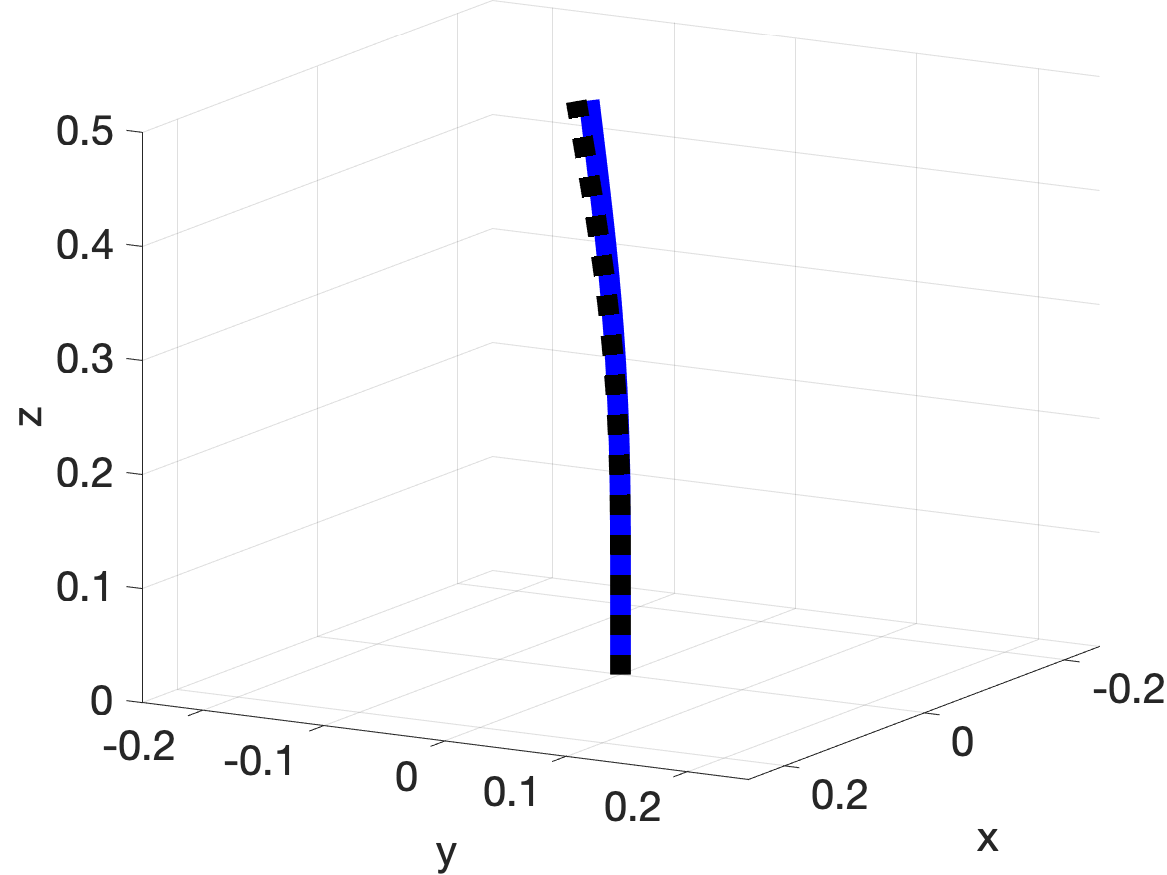}
    \end{subfigure}
    \begin{subfigure}[b]{0.19\textwidth}
        \centering
        \includegraphics[width=\textwidth]{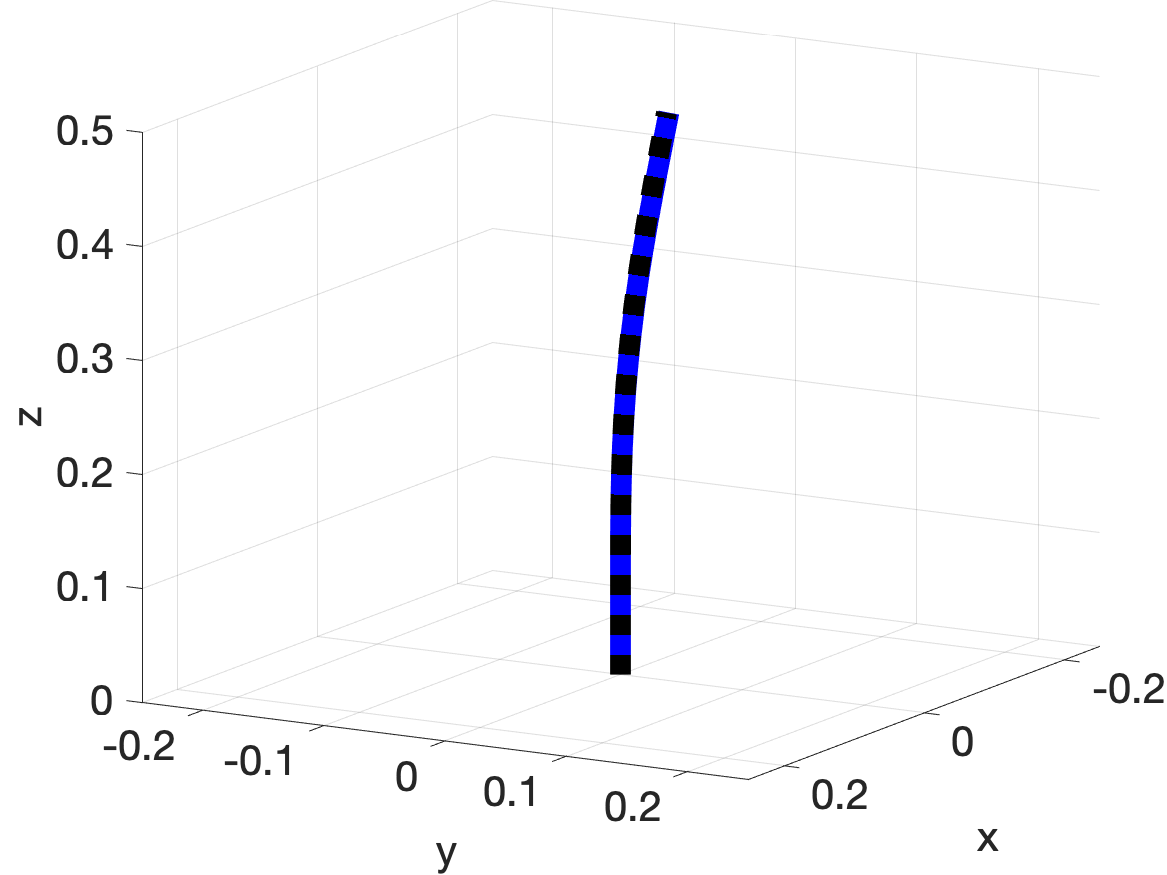}
    \end{subfigure}
    \begin{subfigure}[b]{0.19\textwidth}
        \centering
        \includegraphics[width=\textwidth]{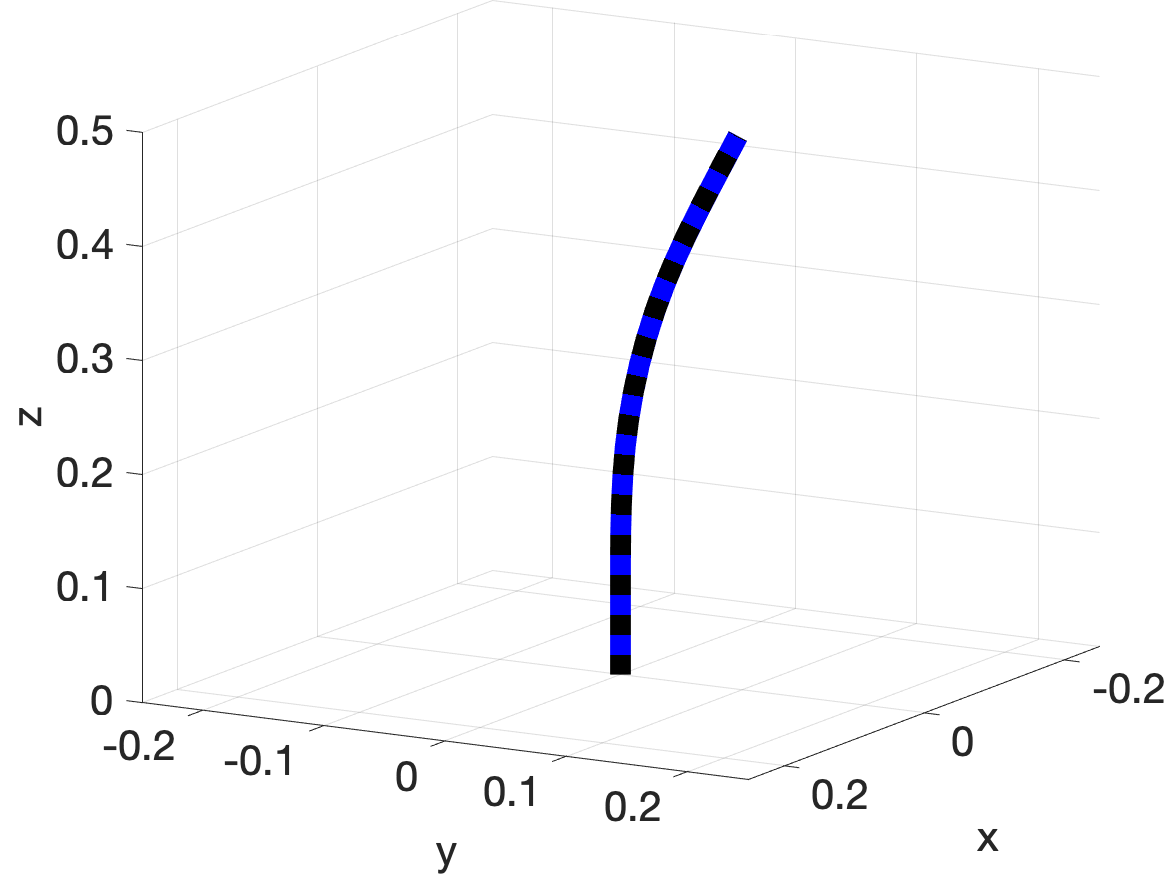}
    \end{subfigure}
    \caption{Configuration tracking. Black dotted: desired configuration. Blue solid: actual configuration.}
    \label{fig:tracking}
\end{figure*}

\begin{remark}
The infinite-dimensional extended Kalman filter can be numerically implemented using approximations such as finite difference or finite element methods.
The computation is fast in general because the resulting matrix representation of $\mathcal{A}_{\hat{\xi}}$ is highly sparse \cite{zheng2021distributedmean}.
% While one may claim that the traditional AtD approach also achieves asymptotic precision with increasing approximations, we point out that the asymptotic property of the DtA approach is stronger.
% In particular, infinite-dimensional analysis allows convergence study on the higher-order terms (like $\hat{p}_s$ and $\hat{R}_s$), which is significant because the feedback control performance of PDE systems usually depends on these derivatives.
% Following a DtA approach, we ensure that increasing approximations also achieve asymptotically accurate estimates on those spatial derivatives, while in the AtD approach this information is lost in the first place.
% The stability study is the subject of continuing work.
\end{remark}

\section{Simulation study}
\label{section:simulation}

To verify the effectiveness of the proposed algorithms, a simulation study is performed on MATLAB.
The system parameters are chosen as $L=0.5\mathrm{m}$, $r=0.02\mathrm{m}$ (radius of cross-section), $\rho=2000\mathrm{kg/m^3}$, $E=0.03\mathrm{Gpa}$, and $G=0.01\mathrm{Gpa}$.
The robotic arm is initially undeformed and lies on the $z$-axis.
Hence, the initial conditions are given by $p(s,0)=[0~0~s]^T$, $R(s,0)=I$, $q(s,0)=[0~0~1]^T$, $u(s,0)=[0~0~0]^T$, $v(s,0)=[0~0~1]^T$, and $\omega(s,0)=[0~0~1]^T$.
The PDE is solved using finite difference methods where we set $ds=0.025$ and $dt=0.0002$.
The desired trajectory is a swinging motion on the $yz$-plane (see Fig. \ref{fig:tracking}) and its initial states differ from the actual states.
We assume the output is given by $p(s,t)+w(s,t)$ where $w\sim\mathcal{N}(0,0.02I)$ for $\forall s$.
The other states are estimated using the proposed extended Kalman filter.
Since the initial configuration is undeformed, precise initial estimates are easily obtained.
This ensures that the linearized solution remains close to the nominal solution.

\begin{figure}[hbt!]
    \centering
    \includegraphics[width=0.9\columnwidth]{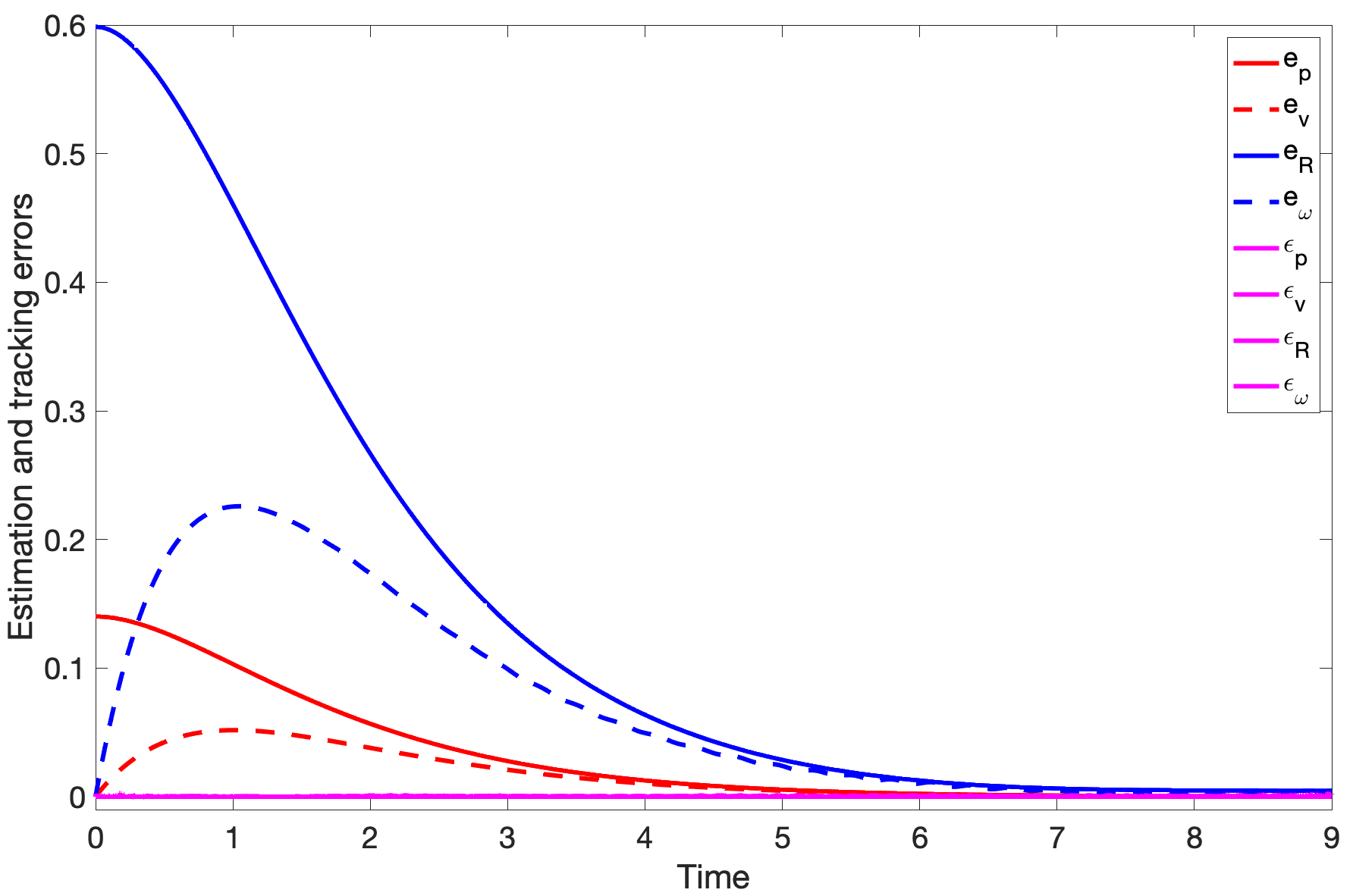}
    \caption{Estimation and tracking errors.}
    \label{fig:errors}
\end{figure}

The estimation errors are denoted by $(\epsilon_p,\epsilon_R,\epsilon_v,\epsilon_\omega)$.
Their sup-norms are plotted in Fig. \ref{fig:errors}, which are observed to remain close to 0 because of the precise initial estimates.
These estimates are used to compute the feedback controller, where the gains are set to be $k_p=k_R=1,k_v=k_\omega=2$.
The sup-norms of the tracking errors are also plotted in Fig. \ref{fig:errors}, which quickly converge to 0.
An illustration of the tracking behavior is given in Fig. \ref{fig:tracking}, where we observe that the initially undeformed robotic arm is able to catch up and eventually keep tracking the desired motion.

\section{Conclusion}
\label{section:conclusion}
In this work, we presented infinite-dimensional state feedback control to achieve trajectory tracking of soft robotic arms that are modeled by nonlinear Cosserat-rod PDEs and proved their uniformly exponential convergence.
To estimate the feedback states, we linearized the PDE system on Lie groups using the exponential map and designed an infinite-dimensional extended Kalman filter to estimate all system states (including position, rotation, curvature, linear and angular velocities) using only position measurements.
These algorithms provided a first step towards the exploitation of infinite-dimensional control and estimation theory for soft robotic systems.
Our future work is to extend the control design when the system is underactuated and study the convergence property of the proposed extended Kalman filter.

\section*{Appendix: linearization on $SO(3)$}
% Consider \eqref{eq:PDE system}.
Based on $R'\approx R+R\delta\eta^\wedge$, we have
\begin{align*}
    R'_s & \approx\big(R(I+\delta\eta^\wedge)\big)_s=R_s+R_s\delta\eta^\wedge+R\delta\eta_s^\wedge=:R_s+\delta R_s \\
    q' & =R'^Tp'_s\approx(I-\delta\eta^\wedge)R^T(p_s+\delta p_s) \\
    & =q+q^\wedge\delta\eta+R^T\delta p_s=:q+\delta q \\
    q'_s & \approx q_s+q_s^\wedge\delta\eta+q^\wedge\delta\eta_s+R_s^T\delta p_s+R^T\delta p_{ss}=:q_s+\delta q_s\\
    u' & =(R'^TR'_s)^\vee\approx\big((I-\delta\eta^\wedge)R^T(R_s+R_s\delta\eta^\wedge+R\delta\eta_s^\wedge)\big)^\vee \\
    & =(u^\wedge+u^\wedge\delta\eta^\wedge+\delta\eta_s^\wedge-\delta\eta^\wedge u^\wedge)^\vee=u+u^\wedge\delta\eta+\delta\eta_s \\
    & =:u+\delta u, \\
    u'_s & \approx u_s+u_s^\wedge\delta\eta+u^\wedge\delta\eta_s+\delta\eta_{ss}=:u_s+\delta u_s.
\end{align*}
Using the chain rule and the pre-computed terms, we derive the following linearized PDE for the perturbations of \eqref{eq:PDE system}:
\begin{align*}
\begin{split}
    \delta p_t= & \delta v \\
    \delta\eta_t= & -\omega^\wedge\delta\eta+\delta\omega \\
    \delta v_t= & \frac{1}{\rho\sigma}\Big[\delta RK_l(q_s-\bar{q}_s)+RK_l\delta q_s \\
    & +\delta R_sK_l(q-\bar{q})+R_sK_l\delta q+\delta f_c\Big] \\
    = & \frac{1}{\rho\sigma}\Big[-R\big(K_l(q_s-\bar{q}_s)\big)^\wedge\delta\eta \\
    & +RK_l(q_s^\wedge\delta\eta+q^\wedge\delta\eta_s+R_s^T\delta p_s+R^T\delta p_{ss}) \\
    & -R_s\big(K_l(q-\bar{q})\big)^\wedge\delta\eta-R\big(K_l(q-\bar{q})\big)^\wedge\delta\eta_s \\
    & +R_sK_l(q^\wedge\delta\eta+R^T\delta p_s)+\delta f_c\Big]
\end{split}
\end{align*}
\begin{align*}
\begin{split}
    \delta\omega_t= & (\rho J)^{-1}\Big[K_a\delta u_s-\big(K_a(u-\bar{u})\big)^\wedge\delta u+u^\wedge K_a\delta u \\
    & -\big(K_l(q-\bar{q})\big)^\wedge\delta q+q^\wedge K_l\delta q+(\rho J\omega)^\wedge\delta\omega \\
    & -\omega^\wedge\rho J\delta\omega+\delta R^T(l_e+l_c)+R^T\delta l_c\Big] \\
    = & (\rho J)^{-1}\Big[K_a(u_s^\wedge\delta\eta+u^\wedge\delta\eta_s+\delta\eta_{ss}) \\
    & +\Big(u^\wedge K_a-\big(K_a(u-\bar{u})\big)^\wedge\Big)(u^\wedge\delta\eta+\delta\eta_s) \\
    & +\Big(q^\wedge K_l-\big(K_l(q-\bar{q})\big)^\wedge\Big)(q^\wedge\delta\eta+R^T\delta p_s) \\
    & +\big((\rho J\omega)^\wedge-\omega^\wedge\rho J\big)\delta\omega+\big(R^T(l_e+l_c)\big)^\wedge\delta\eta+R^T\delta l_c\Big]
\end{split}
\end{align*}
where $R$, $q$, and $u$ are understood as functions of $p$ and $\eta$.

Denote $\xi=(p,\eta,v,\omega)\in\mathbb{R}^{12}$.
To obtain a compact representation, define the following linear operators whose values depend on $\xi$ (and keep in mind that differentiation is a linear operator):
\begin{align*}
    \mathcal{A}_\xi & =\begin{bmatrix}
    0 & 0 & \mathcal{I} & 0 \\
    0 & -\omega^\wedge & 0 & \mathcal{I} \\
    \mathcal{A}_{31} & \mathcal{A}_{32} & 0 & 0 \\
    \mathcal{A}_{41} & \mathcal{A}_{42} & 0 & \mathcal{A}_{44}
    \end{bmatrix}, \\
    \mathcal{C} & =\begin{bmatrix}
    \mathcal{I} & 0 & 0 & 0
    \end{bmatrix}
\end{align*}
where $\mathcal{I}$ is the identity operator and
\begin{align*}
    % \mathcal{A}_{13}= & \mathcal{I}, \quad \mathcal{A}_{22}=-\omega^\wedge, \quad \mathcal{A}_{24}=\mathcal{I}, \\
    \mathcal{A}_{31}= & \frac{1}{\rho\sigma}(RK_lR^T\partial_{ss}+RK_lR_s^T\partial_s+R_sK_lR^T\partial_s), \\
    \mathcal{A}_{32}= & \frac{1}{\rho\sigma}\Big[RK_lq^\wedge\partial_s-R\big(K_l(q-\bar{q})\big)^\wedge\partial_s \\
    & -R\big(K_l(q_s-\bar{q}_s)\big)^\wedge+RK_lq_s^\wedge \\
    & -R_s\big(K_l(q-\bar{q})\big)^\wedge+R_sK_lq^\wedge\Big], \\
    \mathcal{A}_{41}= & (\rho J)^{-1}\Big[q^\wedge K_lR^T\partial_s-\big(K_l(q-\bar{q})\big)^\wedge R^T\partial_s\Big], \\
    \mathcal{A}_{42}= & (\rho J)^{-1}\Big[K_a\partial_{ss}+K_au^\wedge\partial_s-\big(K_a(u-\bar{u})\big)^\wedge\partial_s \\
    & +u^\wedge K_a\partial_s+K_au_s^\wedge-\big(K_a(u-\bar{u})\big)^\wedge u^\wedge+u^\wedge K_au^\wedge \\
    & -\big(K_l(q-\bar{q})\big)^\wedge q^\wedge+q^\wedge K_lq^\wedge+\big(R^T(l_e+l_c)\big)^\wedge\Big], \\
    \mathcal{A}_{44}= & (\rho J)^{-1}\big[(\rho J\omega)^\wedge-\omega^\wedge\rho J\big].
\end{align*}
% \begin{align*}
%     -\omega^\wedge, \quad \mathcal{A}_{34}=-\nu^\wedge, \quad \mathcal{A}_{35}=\partial_s+u^\wedge, \quad \mathcal{A}_{36}=q^\wedge, \\
%     & \mathcal{A}_{44}=-\omega^\wedge, \quad \mathcal{A}_{46}=\partial_s+u^\wedge, \\
%     & \mathcal{A}_{52}=\frac{1}{\rho\sigma}\big(R^T(f_e+f_c)\big)^\wedge, \quad \mathcal{A}_{53}=\frac{1}{\rho\sigma}(K_l\partial_s+u^\wedge K_l), \\
%     & \mathcal{A}_{54}=-\frac{1}{\rho\sigma}\big(K_l(q-\bar{q})\big)^\wedge, \quad \mathcal{A}_{55}=-\omega^\wedge, \quad \mathcal{A}_{56}=\nu^\wedge, \\
%     & \mathcal{A}_{62}=(\rho J)^{-1}\big(R^T(l_e+l_c)\big)^\wedge, \\
%     & \mathcal{A}_{63}=(\rho J)^{-1}\Big[q^\wedge K_l-\big(K_l(q-\bar{q})\big)^\wedge\Big], \\
%     & \mathcal{A}_{64}=(\rho J)^{-1}\Big[K_a\partial_s-\big(K_a(u-\bar{u})\big)^\wedge+u^\wedge K_a\Big], \\
%     & \mathcal{A}_{66}=J^{-1}(J\omega)^\wedge-J^{-1}\omega^\wedge J.
% \end{align*}

\bibliographystyle{IEEEtran}
\bibliography{References}

\end{document}